\documentclass[twoside,11pt]{article}
\pdfoutput=1
\usepackage{jmlr2e}

\usepackage{geometry}
\usepackage{graphicx}
\usepackage{subfigure} 
\usepackage{amsmath, amssymb}

\usepackage{algorithm}
\usepackage{algorithmic}
\usepackage{booktabs}
\usepackage{tabularx}

\usepackage{hyperref,color}

\def\argmin{{\arg\min}}

\def\bw{\mathbf{w}}
\def\bx{\mathbf{x}}

\def\bbR{\mathbb{R}}

\def\cf{{\it cf.}}

\def\minimize{{\sf minimize}}

\def\norm#1{\left\|#1\right\|}

\def\prox{{\sf prox}}

\def\reals{{\bbR}}

\ShortHeadings{Inexact subsampled proximal Newton-type method}{Liu, Hsieh, Lee and Sun}
\firstpageno{1}

\begin{document}

\title{An inexact subsampled proximal Newton-type method for large-scale machine learning}

\author{\name Xuanqing Liu \email xqliu@ucdavis.edu \\
       \addr Department of Computer Science\\
       University of California\\
       Davis, CA 95616
       \AND
       \name Cho-Jui Hsieh \email chohsieh@ucdavis.edu \\
       \addr Department of Statistics and Computer Science\\
       University of California\\
       Davis, CA 95616
       \AND
       \name Jason D. Lee \email jasonlee@marshall.usc.edu\\
       \addr Marshall School of Business\\
       University of Southern California\\
       Los Angeles, CA 90089
       \AND
       \name Yuekai Sun \email yuekai@umich.edu\\
       \addr Department of Statistics\\
       University of Michigan\\
       Ann Arbor, MI 48109
       }

\editor{}

\maketitle

\begin{abstract} 
We propose a fast proximal Newton-type algorithm for minimizing regularized finite sums that returns an $\epsilon$-suboptimal point in $\tilde{\mathcal{O}}(d(n + \sqrt{\kappa d})\log(\frac{1}{\epsilon}))$ FLOPS, where $n$ is number of samples, $d$ is feature dimension, and $\kappa$
is the condition number. As long as $n > d$, the proposed method is more efficient than state-of-the-art accelerated stochastic first-order methods for non-smooth regularizers which requires $\tilde{\mathcal{O}}(d(n + \sqrt{\kappa n})\log(\frac{1}{\epsilon}))$ FLOPS. The key idea is to form the subsampled Newton subproblem in a way that preserves the finite sum structure of the objective, thereby allowing us to leverage recent developments in stochastic first-order methods to solve the subproblem. 
Experimental results verify that the proposed algorithm outperforms previous algorithms for $\ell_1$-regularized
 logistic regression on real datasets. 
\end{abstract}


\section{Introduction}
\label{sec:intro}

We consider optimization problems of the form
\begin{equation}
\minimize_{w\in\reals^d}\ F(w)\triangleq \underbrace{\sum_{i=1}^n f_i(x_i^\intercal w)}_{f(w)} + R(w),
\label{eq:composite-minimization}
\end{equation}
where the $f_i$'s are smooth, convex loss functions, and $R:\reals^d\to\reals$ is a convex but possibly non-smooth regularizer, we also require the smooth part $f(w)$ to be strongly convex and Lipschitz continuous. Such problems are ubiquitous in machine learning applications, and concrete instances include (regularized) linear-regression and logistic regression.

For \eqref{eq:composite-minimization} with smooth regularizer, 
most of the current state-of-the-art algorithms are accelerated stochastic first-order methods,  which need $\mathcal{O}(d(n + \sqrt{\kappa n})\log(\frac{1}{\epsilon}))$ floating point operations (FLOP's) to return an $\epsilon$-suboptimal point (\cf\ \citep{allen2016katyusha}). A notable exception is LiSSA and its variants by \citep{Agarwal2016}, which is a Newton-type method that only needs $\tilde{\mathcal{O}}(d(n + \sqrt{\kappa d})\log(\frac{1}{\epsilon}))$ FLOPS to return an $\epsilon$-suboptimal point, by convention we use $\tilde{\mathcal{O}}$ to suppress $\log$ factors of $n$, $d$, $\kappa$ etc. As long as $n > d$, LiSSA is more efficient than accelerated stochastic first-order methods. However, it only handles smooth regularizers. 
$\tilde{\mathcal{O}}(d(n + \sqrt{\kappa d})\log(\frac{1}{\epsilon}))$ using second order methods for problems
with non-smooth regularizers. 

In this paper, we propose a Newton-type method for solving \eqref{eq:composite-minimization} that has 
fast rate of convergence. Our convergence rate matches state-of-the-art 
stochastic first-order methods and LiSSA for smooth regularizers, but also accommodates non-smooth regularizers. 
The basic idea is to combine a proximal Newton-type methods with a subsampled Hessian approximation
that preserve the finite sum structure of the smooth part of the objective 
in the Newton subproblem by subsampling. This allows us to leverage 
state-of-the-art stochastic first-order methods to solve the subproblem. 
As we shall see, the proposed method matches the efficiency of LiSSA: 
it needs $\tilde{\mathcal{O}}(d(n + \sqrt{\kappa d})\log(\frac{1}{\epsilon}))$ FLOPS 
to return an $\epsilon$-suboptimal point. Thus, as long as $n > d$, the proposed 
method is more efficient than accelerated stochastic first-order methods. 


The rest of the paper is outlined as follows. We present our main algorithm in Section~\ref{sec:proxnewton} and introduce some related work in Section~\ref{sec:related}. The theoretical analysis is presented
in Section \ref{sec:convergence} and experimental results are in Section \ref{sec:experiments}. 



\section{Subsampled Proximal Newton-type methods}
\label{sec:proxnewton}

\begin{algorithm*}[tb]
\caption{Fast proximal Newton\label{alg:fast_newton_alg}}
\begin{algorithmic}[1]
\STATE{\bfseries Input:} Data pairs $(x_i, y_i)|_{i=1}^n$, $\theta_t\in (0,1]$, $\beta_t=\min\{\theta_t, \frac{1}{3} \}$; Desired precision $\epsilon$.
\STATE{\bfseries Output:} $w^*=\argmin_w F(w)$
\STATE $w_0=\mathbf{0}$;
\FOR{\texttt{t}$=0,1,\dots,N-1$}
\STATE Sample a subset $B\subseteq [n]$ by leverage score sampling defined in \eqref{eq:leverage_score_def}. We need $b=\mathcal{O}(d\log d / (\frac{\beta}{1-\beta})^2)$ samples;
\STATE Calculate the subsampled Hessian $B_t$ by \eqref{eq:subsampled_hessian};
\STATE Solve the quasi-Newton subproblem \eqref{eq:subpb} approximately (using Catalyst+SVRG) to ensure
the convergence condition: 
\begin{equation*}
    \|r_t\|_{B_t}^* \le \theta_t\|v_t\|_{B_t}, 
\end{equation*}
where $r_t$ is the gradient residual defined in \eqref{eq:grad_resid}, $v_t=w_t^+-w_t$ and $w_t^+$ is the solution of subproblem defined in \eqref{eq:subpb}. 
In Section \ref{sec:inner_analysis} we show it only takes constant iterations to ensure this stopping condition.

\STATE Choose step size $\eta_t$ by Theorem \ref{th:converge_phaseI} (for Phase I) and \ref{th:linear-quadratic} (for Phase II);
\STATE Update iterate: $w_{t+1}=w_t+\eta_tv_t$;
\IF{ $F(w_{t+1})-F^*\le \epsilon$ (can be checked by Corollary \ref{co:lambda_subopt})}  
\STATE{Break};
\ENDIF
\ENDFOR
\STATE Return $w_N$.
\end{algorithmic}
\end{algorithm*}

The proposed method is, at its core, a proximal Newton-type method. 
The search directions are found by solving the sub-problem 
\begin{equation}
    w_t^+ \approx \argmin_{w} \underbrace{\nabla f (w_t)^{\intercal} (w-w_t) + \frac{1}{2} (w-w_t)^{\intercal} B_t (w-w_t)+ R(w)}_{f^{\text{sub}}_t(w)},\quad v_t\triangleq w_t^+-w_t,    
    \label{eq:subpb}
\end{equation}
where $f(\cdot)$ is the smooth part of composite function defined in \eqref{eq:composite-minimization}, $B_t \succ 0$ is an positive definite approximation to the Hessian. We see that the objective of the subproblem is obtained by replacing the smooth part of the objective by a quadratic approximation. For this reason, the algorithm is also called a {\it successive quadratic approximation method} \citep{byrd2013inexact}. If there is no regularizer, we see that the method reduces to a Newton-type method for minimizing the smooth part of the objective.

From a theoretical perspective, proximal Newton-type methods are known to inherit the desirable convergence properties of Newton-type methods for minimizing smooth functions. Unfortunately, the high cost of solving \eqref{eq:subpb} has prevented widespread adoption of the methods for large-scale machine learning applications.  

In this paper, we combine sub-sampling and recent advances in stochastic first-order methods to solve the sub-problem efficiently. For problem~\eqref{eq:composite-minimization}, the Hessian can be written as
\begin{equation}
 \nabla^2 f(w) = \sum_{i=1}^n \nabla^2 f_i(w^{\intercal} x_i) = \sum_{i=1}^n f_i''(w^{\intercal} x_i) x_i  x_i^{\intercal}. 
 \label{eq:hessian}
\end{equation}
Let $B\subseteq [n]$ be a 
random sample consists $b$ training instances, we define $B_t$ to be the sub-sampled Hessian:
\begin{equation}
  B_t = \frac{1}{1+\epsilon} \sum_{i\in \mathcal{B}}  \frac{1}{p_i} f_i''(w^{\intercal} x_i) x_i x_i^\intercal,  
  \label{eq:subsampled_hessian}
\end{equation}
where $p_i$ is the sampling probability for $i$-th instance and $\epsilon$ is 
a small constant that depends on the Hessian approximation error. In leverage score sampling, the probability $p_i$ is proportional to the corresponding leverage score of Hessian $\nabla^2 f(w)=X^\intercal DX$ where $X\in\mathbb{R}^{n\times d}$ is the data matrix. Let the $i$-th leverage score be $l_i$ \citep{drineas2012fast}:
\begin{equation}\label{eq:leverage_score_def}
\text{(Leverage score sampling)}\quad l_i=\|U_{(i)}\|_2^2,\quad p_i\propto l_i, 
\end{equation}
where $U_{(i)}$ is the $i$-th row of $U\in \mathbb{R}^{n\times r}$ (in our case, $U\Sigma V^\intercal=X\sqrt{D}$). As we shall see, it is possible to emulate leverage score subsampling in $O(d^\omega\log d + nd)$ FLOP's, where $d^\omega$ is (up to a constant), the computational complexity of matrix multiplication. 
In this case, subproblem~\eqref{eq:subpb} can be rewritten as
\begin{equation}
    \argmin_{w} \sum_{i\in \mathcal{B}} \bigg( \frac{1}{2} f_i''(\tilde{w})((w-\tilde{w})^{\intercal} x_i)^2 + \frac{\nabla f(\tilde{w})^{\intercal} w }{|\mathcal{B}|} \bigg)+ R(w), 
    \label{eq:subpb_sum}
\end{equation}
which is the sum of $b=|\mathcal{B}|$ terms plus regularization. The key benefit of forming the Hessian approximation by subsampling is that the subproblem objective retains the finite sum structure of the objective, thereby allowing us to leverage state-of-the-art stochastic first-order methods to solve the subproblem efficiently. We remark that only the hessian is subsampled, not the gradient. This allows the subproblem to capture the first-order characteristics of the original problem, thereby preserving the fast convergence rate of proximal Newton-type methods. As we shall see, the computational cost of an inexact subsampled proximal Newton method is competitive with that of state-of-the-art stochastic first order methods.

Besides combining subsampling and leveraging state-of-the-art stochastic first-order methods to solve the subproblem, the third idea that is crucial to making the proposed method competitive with stochastic first-order methods is inexact search directions. \citep{lee2014proximal,byrd2013inexact} propose an inexact stopping condition based on the relative lengths of the composite gradient step on the subproblem and original objective. We modify their stopping condition to suit the convergence analysis. Define the gradient residual:
\begin{equation}\label{eq:grad_resid}
r_t \in \nabla f(w_t) + B_t(w_t^+ - w_t) + \partial R(w_t^+),
\end{equation}
where $w_t^+ = w_t + v_t$ and $w_t^+$ is the solution of \eqref{eq:subpb}, so $r_t$ is the residual of the first-order optimality condition of the subproblem, if $r_t=\mathbf{0}$ then $w_t^+$ is the exact solution to the subproblem \eqref{eq:subpb}. However we only require
\[
\|r_t\|_{B_t}^* \le (1-\theta_t)\|v_t\|_{B_t},
\]
where $\|\cdot\|_{B_t}$ is the norm induced by $B_t$ and $\|\cdot\|_{B_t}^*$ is its dual norm (equivalently the norm induced by $B_t^{-1}$), $\theta_t\in (0,1]$ is a pre-determined control series. As long as the eigenvalues of $B_t$ remain bounded, the proposed inexact stopping condition is (up to a constant) equivalently to the inexact stopping condition of \citep{lee2014proximal}.

To check the inexact stopping condition, we need a more tractable formulation to compute gradient residual $r_t$ given $v_t$. To do so imagine we do one proximal gradient (PG) step (of step size $\alpha$) on the subproblem \eqref{eq:subpb}:
\[
  v = \prox_{\alpha R}(v_t - \alpha (\nabla f(w_t) + B_tv_t)).
\]
where $v_t$ is the iterate that induces $r_t$ by \eqref{eq:grad_resid}, then by the properties of the proximal mapping, we have
  \[\textstyle
  \frac{1}{\alpha}(v_t - v) \in \nabla f(w_t) + B_tv_t + \partial R(w_t+v_t).
  \]
  
  By subtracting from $B_t(v_t-v)$ at both sides, we see that $({\textstyle\frac{1}{\alpha }\mathbb{I}_d - B_t})(v_t - v)$ is the residual in the inexact stopping condition:
  \begin{equation}
  ({\textstyle\frac{1}{\alpha }\mathbb{I}_d - B_t})(v_t - v) \in \nabla f(w_t) + B_tv + \partial R(w_t+v_t).
  \label{eq:rt}
  \end{equation}

In order to study the computational complexity of the proposed method, we analyze the convergence rate of inexact proximal Newton-type methods on self-concordant composite minimization problems, which may be of independent interest.  Compared to the analysis of \citep{tran2015composite}, our analysis does not rely on an infeasible choice of step size that require evaluating the proximal Newton decrement. 


Our proposed algorithm is summarized in Algorithm~\ref{alg:fast_newton_alg}. Note that in our algorithm, $\beta_t\leq \min\{\theta_t, \frac{1}{3}\}$ controls the inexactness of 
Hessian approximation, and $\theta_t \in (0, 1]$ controls the inexactness of Newton subproblem solver.
Since our analysis holds for any $\beta_t, \theta_t$ satisfy these constraints, 
in practice we can simply choose them to be constants. 

Here are some more details for each step in Algorithm~\ref{alg:fast_newton_alg}: 
\begin{enumerate}
  \item With the leverage score sampling, Theorem~\ref{th:sampling-complexity} and Proposition \ref{pr:dennis_more_deviation}, \ref{pr:equivalence} show that $b=\mathcal{O}(d\log d/(\frac{\beta_t}{1-\beta_t})^2)$ samples are sufficient to guarantee that the subsampled Hessian~\eqref{eq:subsampled_hessian} satisfies Assumption \ref{ass:dennis-more} and \ref{ass:dual_norm} with high probability. In this case, we will show that the subsampled Hessian
  is close enough to the exact Hessian such that an inexact proximal Newton method achieves a linear convergence rate. 
  \item When forming the subsampled Hessian~\eqref{eq:subsampled_hessian}, all we need to do is to calculate $f''_i(w^{\intercal}x_i)$. There is no need for explicitly forming the $d$-by-$d$ matrix $B_t$, since the subproblem solver will directly solve the resulting finite-sum problem. 
  \item Instead of solving the subproblem exactly, we only require an inexact 
    search direction up to a certain precision (controlled by $\theta_t$). As we shall see, by initializing the subproblem solver at the previous solution $w_t$, it is possible to obtain an inexact search direction of sufficient accuracy in a constant number of SVRG+Catalyst iterations. 
  \item To determine the step size $\eta_t$, we first calculate the proximal Newton decrement defined as
    \begin{equation}
      \tilde{\lambda}_t = \|w_t^+ -w_t\|_{B_t}, 
      \label{eq:decrement}
    \end{equation}
    which is the ``inexact'' Newton decrement computed by the inexact 
    subproblem solution and $B_t$ is the approximate Hessian. 
    If $\lambda_t\ge \frac{1}{\sqrt{1+\beta_t}}\tilde{\lambda}_t > \bar{\lambda}$, where $\bar{\lambda}$ is predefined constant, our algorithm is in Phase I and 
    we choose step size $\eta_t$ according to Theorem~\ref{th:converge_phaseI}. Otherwise
    our algorithm is in Phase II and we choose step size $\eta_t=1$ according to Theorem~\ref{th:linear-quadratic}. 
\end{enumerate}

\section{Related work}
\label{sec:related}

Existing algorithms for minimizing composite problems are fall into two broad classes: first-order methods and Newton-type (second-order) methods. 

\subsection{First-order methods}

First-order methods are dominant in large-scale optimization due to the fact that their memory requirement is $\mathcal{O}(d)$, where $d$ is the problem dimension. The basic variants of most first-order methods converge linearly on strongly convex objectives, and their rate of convergence depends on the condition number $\kappa$ of the objective. The accelerated variants improve the dependence on the condition number to $\sqrt{\kappa}$ \citep{nesterov2004introductory}. Broadly speaking, to produce a $\epsilon$-suboptimal iterate, first-order methods require $\mathcal{O}(\kappa nd\log(\frac{1}{\epsilon}))$ floating point operations (FLOP's), while their accelerated counterparts need $\mathcal{O}(\sqrt{\kappa}nd\log(\frac{1}{\epsilon}))$. 

On large $n$ problems, stochastic first-order methods are preferred because they need fewer passes over the data than their non-stochastic counterparts \citep{Robbins1951}. Recently, the idea of {\it variance reduction} has led to significant improvements in the efficiency of stochastic first-order methods \citep{johnson2013accelerating,Xiao2014proximal,Roux2012,Defazio2014,Shalev-Shwartz2013}. The key idea is to compute the gradient of the objective sparingly during optimization to reduce the variance of the steps as the algorithm converges. The resulting algorithms achieve linear rates of convergence that are comparable to those of their non-stochastic counterparts. Broadly speaking, these methods reduce the computational cost of obtaining a $\epsilon$-suboptimal point to $\mathcal{O}((n + \kappa)d\log(\frac{1}{\epsilon}))$ FLOP's. Accelerated variants of such stochastic first-order method with variance reduction further reduce the cost to $\mathcal{O}((n + \sqrt{\kappa n})d\log(\frac{1}{\epsilon}))$ \citep{Allen-Zhu2016,Lin2015,Shalev-Shwartz2014}.

\subsection{Newton-type methods}

Traditional second-order methods, due to their higher computational cost, have been relegated to medium-scale problems. 
The main bottleneck is forming the $d$-by-$d$ Hessian matrix and computing the Newton direction by solving an $d$-by-$d$ linear system. 
Conjugate gradient method can be used to accelerate this procedure by solving the linear system inexactly, and has been successfully
used in some machine learning tasks~\citep{lin2008trust,keerthi2005modified}. 

For problems with non-smooth regularizers (e.g., $\ell_1$ penalty), Newton-type methods cannot be directly applied since the objective is non-differentiable. For these problems, a family of
proximal Newton methods has been studied recently~\citep{lee2014proximal}. 
To deal with non-smooth regularizers, proximal Newton methods compute the Newton direction by solving a quadratic plus non-smooth subproblem
which does not have a closed form solution, so another iterative solver has to be used to solve the subproblem approximately. 
There are a few specialized proximal Newton algorithms tailored to specific problems, that achieve state-of-the-art performance \citep{hsieh2011sparse,yuan2012improved,friedman2007pathwise}. 

Recently, there has been a line of research that aims to reduce the computational cost of Newton-type methods so that they are competitive with state-of-the-art first-order methods. The key ideas here are subsampling and exploiting the low-rank structure in limited-memory Hessian approximations to accelerate the solution of the Newton subproblem \citep{Erdogdu2015,Agarwal2016,Byrd2011,roosta2016sub1,roosta2016sub2,xu2016sub,ye2016revisiting,pilanci2015newton}. \citep{Erdogdu2015} uses uniform subsampled Hessian called NewSamp, and \citep{pilanci2015newton} uses sketching in place of subsampling. Non-uniform subsampling and especially leverage sampling is desirable because they require a subsample size that does not depend on $n$. For example \citep{xu2016sub} uses blocked partial leverage scores to sample $\mathcal{O}(d\log d)$ data points in $\mathcal{O}(\textbf{nnz}(X)\log n)$ FLOPS, but the overall method only handles with smooth problems. 
At the same time, \citep{roosta2016sub1,roosta2016sub2} use both Hessian and gradient uniform sampling scheme to reach a better per iteration cost when $n\gg p\gg 1$. 
\par
Unfortunately, all the previous work focus on smooth functions where
the Newton direction can be computed in closed form or by solving a linear system.
Compared with prior work, our method is the first subsampled second order method for problems with non-smooth regularizers. To deal with non-smooth regularizers, we appeal to the proximal Newton framework. 
Since the proximal Newton subproblem is itself non-smooth 
and does not have a closed form solution, we use another iterative solver to compute search directions, which leads us to the question of how to balance the inexactness of subproblem solvers and computational cost. Our theoretical analysis
shows that the convergence rate of the proposed second order method has better computational complexity than state-of-the-art first order methods for solving
\eqref{eq:composite-minimization} with non-smooth regularizers. 

In a related area, there has also been considerable research on stochastic Newton-type methods that aim to incorporate second order information into stochastic first-order methods \citep{Byrd2016,Schraudolph2007}. Unfortunately these methods generally retain the sublinear convergence rate of their first-order counterparts. The exception to this is the algorithm by, which attain a linear rate of convergence \citep{Moritz2016}. Unfortunately, the rate of convergence depends poorly on the condition number of the objective.

\subsection{Leverage score subsampling}
In this section we will introduce the fast leverage score subsampling algorithm of \citep{cohen2015uniform}. This algorithm is used for forming the Hessian approximation in Algorithm \ref{alg:fast_newton_alg}.

\begin{theorem}[\citet{cohen2015uniform}]\label{th:sampling-complexity}
Given a matrix $A$, we can compute a matrix $\tilde A$ with $\mathcal{O}(\frac{d}{\epsilon^2} \log d)$ rows such that for all $x$, 
$$
\frac{1}{1+\epsilon} \norm{Ax}^2 \le \norm{\tilde Ax}^2 \le \norm{Ax}^2 ,
$$
in $\mathcal{O}(nnz(A) +d^\omega \log^2 d +\frac{d^{2.01}}{\epsilon^2})$ time\footnote{$d^\omega$ is the running time of matrix multiplication, so $\omega\le 2.38$. }. 
\end{theorem}

The highlight of \citep{cohen2015uniform} is it doesn't seek to the exact leverage scores in one shot, rather, it adopts an iterative 
scheme: first sample a subset  of $A$ uniformly to construct a crude spectral approximation, then resample rows from $A$ using these estimates to get a finer estimation. By repeating this process iteratively we can find the spectral approximation of $A$ within desired error tolerance. 

\subsection{Catalyst and SVRG}
The Catalyst procedure or accelerated proximal point algorithm is a continuation technique for improving the computational complexity of optimization algorithms.
At each step, Catalyst adds a small strongly convex term $\frac{\kappa}{2}\|x-y_i\|^2$ to the
objective function, thereby making it easier to solve, and solves the modified problem using a first order method. By carefully controlling the
decrease of $\kappa$, \cite{Lin2015Universal} showed that the convergence rate can be improved. The algorithm for Catalyst with SVRG is listed in Algorithm~\ref{algo}. 

\begin{algorithm}\label{algo}
  \caption{Universal Catalyst with SVRG solver}
  \begin{algorithmic}[1]
    \STATE Input: ($f$,$\zeta$,$x_0$)
    \STATE $q\leftarrow\frac{\mu}{\mu+\zeta}$, $\alpha_0\leftarrow\sqrt{q}$, $\gamma_0\leftarrow \beta_0$
    \FOR {$i\leftarrow 0$ to $N$}
    \STATE $x_{i+1}\leftarrow$ SVRG($f(x) + \frac{\zeta}{2} \norm{x- y_i}^2$, \texttt{init}=$x_i$).
    \STATE Solve $\alpha_{i+1}^2\leftarrow(1-\alpha_{i+1})\alpha_i^2+q\alpha_i$
    \STATE $y_{i+1}\leftarrow x_{i+1}+\frac{\alpha_i(1-\alpha_i)}{\alpha_i^2+\alpha_{i+1}}(x_{i+1}-x_i)$
    \ENDFOR
    \STATE Return $x_N$ 
  \end{algorithmic}
\end{algorithm}

The following theorem shows that Catalyst+SVRG converges linearly. In this paper, we will use this algorithm to solve the proximal Newton subproblem~\eqref{eq:subpb}. 
\begin{theorem}[\cite{Lin2015Universal}]\label{th:convergence-catalyst}
Choose $\zeta=\frac{L}{b}-\mu$ (parameter $\kappa$ in \cite{Lin2015Universal}) and use SVRG to solve the subproblem, then the function value of subproblem decreases linearly and find a $\epsilon$-suboptimal solution within $\tilde{\mathcal{O}}(\sqrt{b\kappa}\log(\frac{\epsilon_0}{\epsilon}))$ steps, $\tilde{\mathcal{O}}$ omits  constants and log factors of $b$ and $\kappa$. Formally:
$$
f^{\text{sub}}_t(w_T)\le \min_w f^{\text{sub}}_t(w)+\epsilon,\ \ \text{as long as: } T\ge \tilde{\mathcal{O}}(\sqrt{b\kappa}\log(\epsilon_0/\epsilon)).
$$
\end{theorem}

In the next section, we present a convergence analysis of inexact proximal Newton-type methods on self-concordant composite minimization problems, which may be of independent interest.

\section{Convergence analysis}
\label{sec:convergence}

\subsection{Preliminaries on self-concordant function}
In this paper we focus on composite objectives where the smooth part is self-concordants. 
\begin{definition}
\label{def:self-concordance}
(Self-concordant) A closed, convex function $f: \mathbb{R}^d\rightarrow \mathbb{R}$ is called self-concordant if:
\begin{equation}
\frac{d}{d\alpha}\nabla^2 f(x+\alpha v)|_{\alpha=0}\preceq 2\Vert v\Vert_x\nabla^2f(x),
\end{equation}
for all $x\in \mathsf{dom} f$ and $v\in\mathbb{R}^d$, where $\Vert v\Vert_x=(v^\intercal\nabla^2 f(x)v)^{\frac{1}{2}}$ is the local norm.
\end{definition}
We claim that the regularized logistic regression loss $\log(1+\exp(-yw^\intercal x)) +\frac{\gamma}{2} \norm{w}^2$ is self-concordant for any $\gamma>0$ \citep{zhang2015disco}, so we will use that in our experiment. 
For self-concordant function $f$ we have some useful inequalities:
\begin{itemize}
    \item Hessian bound:
    \begin{equation}\label{eq:self-concordance-hessian}
    \begin{aligned}
    &\nabla^2 f(y) \succeq (1-\Vert x-y\Vert_x)^2\nabla^2 f(x),  \\
    &\nabla^2 f(y) \preceq \frac{1}{(1-\Vert x-y\Vert_x)^2}\nabla^2f(x).
    \end{aligned}
    \end{equation}
    \item Gradient bound:
    \begin{equation}
    \Vert\nabla f(y)-\nabla f(x)-\nabla^2f(x)(y-x)\Vert_x^* \le \frac{\Vert x-y \Vert_x^2}{1-\Vert x-y\Vert_x}\label{eq:self-concordance-gradient}.
    \end{equation}
    \item Function value bound:
    \begin{equation}
    \zeta(\Vert x-y \Vert_x)\le f(y)-f(x)-\nabla f(x)^\intercal (y-x)\le \zeta^*(\Vert x-y \Vert_x) \label{eq:self-concordance-lipschitz},
    \end{equation}
\end{itemize}
where $\zeta(x)=x-\log(1+x)$, $\zeta^*(x)=-x-\log(1-x)$. (\ref{eq:self-concordance-hessian},\ref{eq:self-concordance-gradient}) and the right hand side of \eqref{eq:self-concordance-lipschitz} hold for $\Vert x-y \Vert_x<1$. Similar to the global analysis of Newton's method, we divide the convergence analysis into two phases. In the first phase we will show in Section \ref{sec:phase-I} that the objective function value decreases by at least a constant value at each iteration. In the second phase, the objective function value converges to its minimum linearly. Since we use the subsampled Hessian and solve the inner problem inexactly, the following conditions on the inaccuracy of the Hessian approximation are required. These are analogues of the Dennis-Mor\'{e} condition in the analysis of quasi-Newton methods.

\begin{assumption}\label{ass:dennis-more}
(Dennis-Mor\'e condition) For all $v_t\in\mathsf{cone}(\mathbb{B}^d-w_t)$ ($\mathbb{B}^d$ is the unit ball in $\mathbb{R}^d$), the subsampled Hessian matrix $B_t$ satisfies $|v_t^\intercal(B_t-\nabla^2 f(w_t))v_t|\le \beta_t\Vert v_t \Vert^2_{w_t}$ where $\beta_t$ is a parameter that controls the preciseness of $B_t$ (will be fixed later). This is also equivalent to $\Vert v_t\Vert_{w_t}\le \beta_t'\Vert v_t \Vert_{B_t}$, $\beta_t'=\frac{1}{\sqrt{1-\beta_t}}$. 

\end{assumption}
\begin{assumption}\label{ass:dual_norm}
  For all $v\in\mathbb{R}^d$, we have 
  $\Vert (B_t-\nabla^2f(w_t))v \Vert^*_{w_t}\le \beta_t\Vert v \Vert_{w_t}$, 
  where dual norm $\Vert g\Vert_x^*:=\sqrt{g^\intercal \nabla^2 f(x)^{-1} g}$. 
  
\end{assumption}
Next we show that Assumption \ref{ass:dennis-more}(Dennis-Mor\'e condition) implies Assumption \ref{ass:dual_norm}:
\begin{proposition}\label{pr:dennis_more_deviation}
If Hessian approximation $B_t$ satisfies Dennis-Mor\'e condition, then $\Vert (B_t-\nabla^2f(w_t))v \Vert^*_{w_t}\le \beta_t\Vert v \Vert_{w_t}$ will also hold.
\end{proposition}
Further, both Assumption~\ref{ass:dennis-more} and 
\ref{ass:dual_norm} are satisfied if $B_t$ is a good enough spectral approximation of 
$\nabla^2 f(w_t)$: 
\begin{proposition} \label{pr:equivalence}
  To satisfy Assumption~\ref{ass:dennis-more} it is enough to set $\epsilon=\frac{\beta_t}{1-\beta_t}$ in Theorem \ref{th:sampling-complexity}. 
\end{proposition}
Theorem~\ref{th:sampling-complexity} indicates that Assumption~\ref{ass:dennis-more} and \ref{ass:dual_norm} will hold if we form the Hessian approximation $B_t$ defined in~\eqref{eq:subsampled_hessian} using {$\mathcal{O}(d\log d)$} samples, where the probability that a sample is selected is proportional to its leverage score.

Based on the properties of self-concordant functions and the preceding conditions on the Hessian approximation, we are ready to show that the convergence rate is linear. In 
the proof we 
will follow the update rule and notations introduced in Algorithm 
\ref{alg:fast_newton_alg}.
\subsection{Outer loop analysis and stopping criterion\label{sec:phase-I}}
Denote $\lambda_t=\Vert w_t^+-w_t \Vert_{w_t}$ as (exact) proximal Newton decrement and $\tilde{\lambda}_t=\|w_t^+-w_t\|_{B_t}$ as the approximate Newton decrement, from Assumption \ref{ass:dennis-more} we have $\sqrt{1-\beta_t}\lambda_t\le \tilde{\lambda}_t\le \sqrt{1+\beta_t}\lambda_t$. As long as $\lambda_t\ge\bar{\lambda}$, where $\bar{\lambda} > 0$ is a small constant, the algorithm is in phase I. Theorem \ref{th:converge_phaseI} together with Corollary \ref{co:fix_decrement} show that during this phase, the objective value decreases by at least a constant in each iteration.

\begin{theorem}\label{th:converge_phaseI}
By the update rule of Algorithm \ref{alg:fast_newton_alg} with step size: 
\begin{equation*}
   \eta_t\le \frac{1}{1+\beta'_t\tilde{\lambda}_t}, \quad \beta'_t=\frac{1}{\sqrt{1-\beta_t}}, \quad \tilde{\lambda}_t=\Vert w_t^+-w_t \Vert_{B_t},
\end{equation*}
and solve the inner problem with precision $\Vert r_t\Vert_{w_t}^*\le (1-\theta_t)\lambda_t$, where $r_t$ is the subgradient residual:
\begin{equation*}
    \begin{aligned}
    r_t-\nabla f(w_t)-B_t(w_t^+-w_t)\in \partial R(w_t^+),
    \end{aligned}
\end{equation*}
$\theta_t\in(0, 1]$ is a forcing coefficient. Then the function value will decrease by:
\begin{equation}\label{eq:improvement-phaseI}
  F(w_{t+1})\le F(w_t) - \eta_t(\theta_t - \beta_t)\lambda_t^2 + 
  \zeta^*(\eta_t\lambda_t). 
\end{equation}
\end{theorem}
We remark that unlike the step size proposed by \citep{tran2013composite} 
where $\eta_t=\frac{\tilde{\lambda}_t^2}{\lambda_t(\lambda_t+\tilde{\lambda}_t^2)}$, 
our step size does not depend on exact Newton decrement $\lambda_t$ and Hessian
$\nabla f^2(w_t)$. In practical implementations of proximal quasi-Newton methods, calculating 
$\lambda_t$ is impractical. Our step size only depends on the Newton decrement $\tilde{\lambda}_t$, which is available in our algorithm.

\begin{corollary}\label{co:fix_decrement}
By fixing $\beta_t<\min\{\theta_t, \frac{1}{3}\}$ and step size $\eta_t = \frac{\theta_t - \beta_t}{1+\beta_t'(\theta_t-\beta_t)\tilde{\lambda}_t}<\frac{1}{1+\beta_t'\tilde{\lambda}_t}$,  the decrement of function value at each step is at least $(\frac{1}{2(1-\beta_t)}-\frac{2\beta_t}{1-\beta_t^2})\eta_t(\theta_t-\beta_t)\tilde{\lambda}_t^2$ which is bounded away from zero as long as $\lambda_t\ge\bar{\lambda}$. So within finite steps, the iterates will enter into $\lambda_t< \bar{\lambda}$ (defined as Phase-II).
\end{corollary}

When $\lambda_t<\bar{\lambda}$ or equivalently 
$\beta_t'\tilde{\lambda}_t<\bar{\lambda}$ our algorithm switches to undamped 
subsampled proximal Newton method, where step size $\eta_t=1$ is adopted and so $w_{t+1}=w_t^+$. The following theorem 
indicates that the Newton decrement, as a metric of suboptimality, converges to 
zero linear-quadratically:
\begin{theorem}\label{th:linear-quadratic}
When $\lambda_t<\bar{\lambda}$, if step size $\eta_t=1$ and the subproblem solver yields a solution such that the subgradient residual $\Vert r_t\Vert_{w_t}^*\le (1-\theta_t)\lambda_t$, $\theta_t\in(0, 1]$ then Newton decrement $\lambda_t$ will converge to zero linear-quadratically:
\begin{equation}\label{eq:linear_lambda}
\lambda_{t+1}\le \frac{\frac{\theta_t-\beta_t}{\theta_{t+1}-\beta_{t+1}}\lambda_t^2+\frac{1+\beta_t-\theta_t}{\theta_{t+1}-\beta_{t+1}}\lambda_t}{(1-\lambda_t)^2}.
\end{equation}
If $\beta_t\ne 0$, $\theta_t\ne 1$ and $\lambda_t$ is small enough, the numerator of RHS will be dominated by $\frac{1+\beta_t-\theta_t}{\theta_{t+1}-\beta_{t+1}}\lambda_t$ so the contraction factor is $\rho=\frac{1+\beta_t-\theta_t}{\theta_{t+1}-\beta_{t+1}}$ asymptotically.
\end{theorem}

If we set $\beta_t=0$ (so that $B_t$ is exact Hessian) and 
$\theta_t=1$ (so the subproblem is solved exactly), we recover the quadratic convergence rate of the proximal Newton method: 
\eqref{eq:linear_lambda} 
becomes:
\begin{equation*}
\lambda_{t+1}\le \frac{\lambda_t^2}{(1-\lambda_t)^2}.
\end{equation*}

By using the connection between the Newton decrement $\lambda_t$ and suboptimality $F(w_t)-F^*$, we show that the suboptimality also decreases linearly:

\begin{corollary}\label{co:lambda_subopt}
If $\lambda_t<\min\{\bar{\lambda}, \frac{1}{2-\theta_t}\}$ and use the undamped update: $w_{t+1}=w_t^+$, then the function value to minimum is upper bounded by: 
\begin{equation}
F(w_t^+)-F(w^*) \le \lambda_t^2.
\end{equation}
It is easy to see that the $\text{LHS}\to 0$ as $\lambda_t\to 0$. Practically we use $\tilde{\lambda}_t$ to replace $\lambda_t$, this is validated by Dennis-Mor\'e condition \ref{ass:dennis-more}.
\end{corollary}
Note that the proof is similar to the analysis of \citep{li2016inexact} but we modify it to accommodate an inexact Hessian.approximation.

In addition to the convergence rate, Corollary \ref{co:lambda_subopt} gives a stopping criterion for the outer iteration of our algorithm: For any desired error tolerance $\epsilon>0$, we terminate the algorithm as long as $\lambda_t\le \sqrt{\epsilon}$. In practice, we replace $\lambda_t$ by $c\tilde{\lambda}_t$, where $c> 0$ is a small constant. This is justified by the fact that $B_t$ is a good spectral approximation of the exact Hessian.

\subsection{Inner loop analysis\label{sec:inner_analysis}}
In this part we show that to satisfy the precision requirement 
$\Vert r_t\Vert_{w_t}^*\le (1-\theta_t)\lambda_t$ we only need a constant 
number of  
inner iterations if we use variance reduction method such as SVRG and it can 
be further accelerated by Catalyst~\citep{Lin2015Universal}. Theorem \ref{th:convergence-catalyst} indicates that catalyst can accelerate many first order methods like SVRG to change the dependent of condition number from $\kappa$ to $\sqrt{\kappa}$.


Recall the solution of the subproblem should satisfy the inexact stopping condition $\|r_t\|_{w_t}^* \le (1-\theta_t)\lambda_t$ (see 
Theorem~\ref{th:converge_phaseI} and Theorem~\ref{th:linear-quadratic}). The following 
lemma converts the condition on $\|r_t\|_{w_t}^*$ to the function value of subproblem.
\begin{lemma}\label{Lemma:precise}
To satisfy the condition $\Vert r_t\Vert_{w_t}^*\le (1-\theta_t)\lambda_t$ it is enough to solve the subproblem to a certain precision defined below:
\begin{equation}\label{eq:subproblem-precision}
f^{\text{sub}}_t(w_t^+)-f_t^*\le \frac{\mu L}{2(L^2-\mu^2)}((1-\theta_t)\lambda_t)^2, 
\end{equation}
where $f^{\text{sub}}_t(w)$ introduced in \eqref{eq:subpb} is the subproblem at $t$-th outer iteration:
\begin{equation}\label{eq:subproblem}
f^{\text{sub}}_t(w)=\nabla f^{\intercal}(w_t)+\frac{1}{2}(w-w_t)^{\intercal}B_t(w-w_t)+R(w),
\end{equation}
and $f_t^*:=\min_w f^{\text{sub}}_t(w)$ is its minimum.
\end{lemma}

Since the proximal Newton decrement converges to zero linearly in phase II, the 
number of inner iterations should increase linearly. However, if we use the 
last iterate as the initial guess to ``warm start'' the subproblem solution, 
we only need a constant number of iterations each time. 

\begin{lemma}\label{le:warm_start}
With the definition of subproblem $f^{\text{sub}}_t(w)$ in \eqref{eq:subproblem}, suppose 
$w_{t}^+$ is the $\epsilon_t$-inexact solution of $f^{\text{sub}}_t(w)$ that satisfies \eqref{eq:subproblem-precision}, i.e.
\begin{equation}
f^{\text{sub}}_t(w_t^+)-f_t^* = \epsilon_t \le \frac{\mu 
L}{2(L^2-\mu^2)}((1-\theta_t)\lambda_t)^2.
\end{equation}
If we initialize the next subproblem $\min_w f^{\text{sub}}_{t+1}(w)$ with $w_{\text{init}}=w_t^+$ then the initial error $f^{\text{sub}}_{t+1}(w_t^+)-f_{t+1}^*$ has the same order of magnitude with desired error. That is,
$$
f^{\text{sub}}_{t+1}(w_t^+)-f_{t+1}^*\le c\cdot \epsilon_t=\mathcal{O}(\lambda_{t+1}^2),
$$
where $c > 0$ is a constant that does not change across major iterations of the subsampled proximal Newton method.
\end{lemma}
Note that many stochastic first order method such as SVRG is only guaranteed to find an $\epsilon_t$-optimal solution with certain probability. 
While as the number of outer iteration 
grows, number of SVRG calling also increases, so we need to make sure each 
SVRG calling successes with high enough probability such that the total process 
success. Specifically we use the following union bound property: 
$P(\bigcup_{t=1}^mA_t)\le \sum_{t=1}^mP(A_t)$ where 
$A_t=\{f^{\text{sub}}_t(w_t^+)-f_t^*>\epsilon_t\}$ is the incident that the $i$-th 
subproblem fails to converge within given iterations and $m$ is the number of outer 
iterations. By Markov inequality the failure probability of each SVRG calling
is bounded by:
\begin{equation}\label{eq:bound-failure-rate}
p=\mathbb{P}[f^{\text{sub}}_t(w_t^+)-f_t^*>\epsilon_t]<\frac{\mathbb{E}f^{\text{sub}}_t(w_t^+)-f_t^*}{\epsilon_t}.
\end{equation}
Therefore, if we desire the overall probability of failure to be at most $p$, it suffices to make sure failure rate small enough:
\begin{equation}\label{eq:union-bound-failure}
  \frac{p}{m}\ge \frac{\mathbb{E}f^{\text{sub}}_t(w_t^+)-f_t^*}{\epsilon_t} > 
  \mathbb{P}[f^{\text{sub}}_t(w_t^+)-f_t^*>\epsilon_t]. 
\end{equation}
Combining the above inequalities, we obtain the following theorem showing the overall complexity which includes sampling overhead, inner loop 
and outer loop complexity.
\begin{theorem}\label{th:final-complexity}
Our fast inexact proximal Newton method, with Catalyst and SVRG as inner solver, can find an $\epsilon$-optimal solution with probability $1-\delta$ within $\mathcal{O}\Big(\log(1/\epsilon)\big(\mathbf{nnz}(X)+d\sqrt{b\kappa}\log(\frac{\log 1/\epsilon}{\delta})\big)\Big)$ time, where $b=d\log d$ is sample size. The complexity of leverage sampling is simplified from Lemma 10 in \citep{cohen2015uniform}, when $n>d^{\omega-1}$ .
\end{theorem}
\begin{proof}
We outline the proof of Theorem \ref{th:final-complexity} here. First of all from Theorem \ref{co:fix_decrement} we know that the iterate will reach $\lambda_t < \bar{\lambda}$ within $\mathcal{O}(\frac{F(w_0)-F(w^*)}{\inf_t Z(\eta_t)})$ where $Z(\eta_t)=(\frac{1}{2(1-\beta_t)}-\frac{2\beta_t}{1-\beta_t^2})\eta_t(\theta_t-\beta_t)\tilde{\lambda}_t^2$ is the lower bound of function decrement introduced in Corollary \ref{co:fix_decrement}. As long as $F(w^*)$ is bounded below, phase-II will be reached within constant iterations. From Theorem \ref{th:linear-quadratic} we know $\lambda_t$ decrease linearly and from Corollary \ref{co:lambda_subopt} the algorithm can exit when $\lambda_t\le \sqrt{\frac{\epsilon}{1+\beta_t}}$ so the number outer iterations in phase-II is:
$$
m=\mathcal{O}(\frac{\log\epsilon}{\log\rho}), 
$$
where $\rho=\frac{1+\beta_t-\theta_t}{\theta_{t+1}-\beta_{t+1}}$ is the linear convergence rate of $\lambda_t$.
\par
For each major iteration, as long as $n>d^{\omega-1}$, the cost of subsampling is $\mathcal{O}(\mathbf{nnz}(X))$ FLOPS according to Theorem \ref{th:sampling-complexity}. Further, by the union bound \eqref{eq:union-bound-failure} and Lemma \ref{le:warm_start}, using Catalyst and SVRG we have an upper bound on number of inner iterations:
$$
\text{\#inner}=\mathcal{O}(\sqrt{b\kappa}\log\frac{m}{\delta}). 
$$
Finally we combine outer loop complexity with inner loop complexity and notice that for each inner iteration it takes $\mathcal{O}(d)$ FLOPS:
\begin{equation}
  \begin{aligned}
   \mathcal{O}(m(\mathbf{nnz}(X)+d\sqrt{b\kappa}\log\frac{m}{\delta}))
  =\mathcal{O}(\log{(1/\epsilon)}(\mathbf{nnz}(X)+d\sqrt{b\kappa}\log(\frac{\log(1/\epsilon)}{\delta}))). 
  \end{aligned}
\end{equation}
\end{proof}


\section{Experiments}
\label{sec:experiments}

In this section, we compare our proposed algorithm with other 1st/2nd order methods on $\ell_1$-regularized logistic 
regression problem: 
\begin{equation*}
  \bw^*=\argmin_{\bw} \frac{1}{n}\sum_{i=1}^n (\log(1+e^{-y_i \bw^\intercal \bx_i})) + \lambda 
  \|\bw\|_1, 
\end{equation*}
where $\{(\bx_i, y_i)\}_{i=1}^n$ are training data/label pairs and $\lambda$
is the regularization parameter.
Three datasets from LIBSVM website are chosen. 
Because Mnist8M is a multiclass dataset, we extract the 1st and 6th classes to synthesize a two-class dataset. Other basic information about datasets is listed in Table~\ref{tab:datasets}.
These three datasets mainly differ in sparsity which we believe is an important factor when comparing different algorithms. 

\begin{table*}[!htbp]
  \centering
  \caption{Dataset Statistics and Parameters Used in Experiments \label{tab:datasets}}
  \begin{tabular}{crrr}
    \hline 
    Dataset & \#Data & \#Features & \#Non-zeros  \\ 
    \hline
    Realsim & 72,309 & 20,958 & 3,781,392  \\\hline
    Covtype & 581,012 & 54 & 7,521,450 \\\hline
    Mnist8M & 1,603,260 & 784 & 345,075,085  \\\hline
  \end{tabular}
\end{table*}

We compare the following algorithms with our fast proximal Newton method: 
\begin{itemize}
\item LIBLINEAR (Full proximal Newton): The proximal Newton method with exact Hessian is used as the default solver
for $\ell_1$ logistic regression in LIBLINEAR~\citep{fan2008liblinear}. 
\item SVRG: the variance reduced SGD algorithm proposed 
  in~\citep{johnson2013accelerating}. 
\item SAGA: another variance reduced SGD algorithm proposed 
  in~\citep{Defazio2014}. Our implementation uses $O(1)$ storage per sample by exploiting the 
  structure of the ERM problem. 
\end{itemize}
Since the notion of ``epoch'' is quite different for these algorithms, unlike many other experiments which use data passes or gradient calculation as x-axis, we evaluate performance by comparing running time of different methods. We implement all the algorithms
in C++ by modifying the code base of LIBLINEAR, and try to optimize each of them in order to have a fair comparison. 
In the following, we first test our algorithm with different parameter settings, and then compare it with other competing algorithms. 

In the first set of experiments we consider how number of inner iterations affects the convergence rate, 
by setting number of inner iteration($\texttt{inner})=1,2,...,6$ we can observe the convergence rates in Figure \ref{Fig:inner}.

\begin{figure*}
\centering
\includegraphics[width=5in]{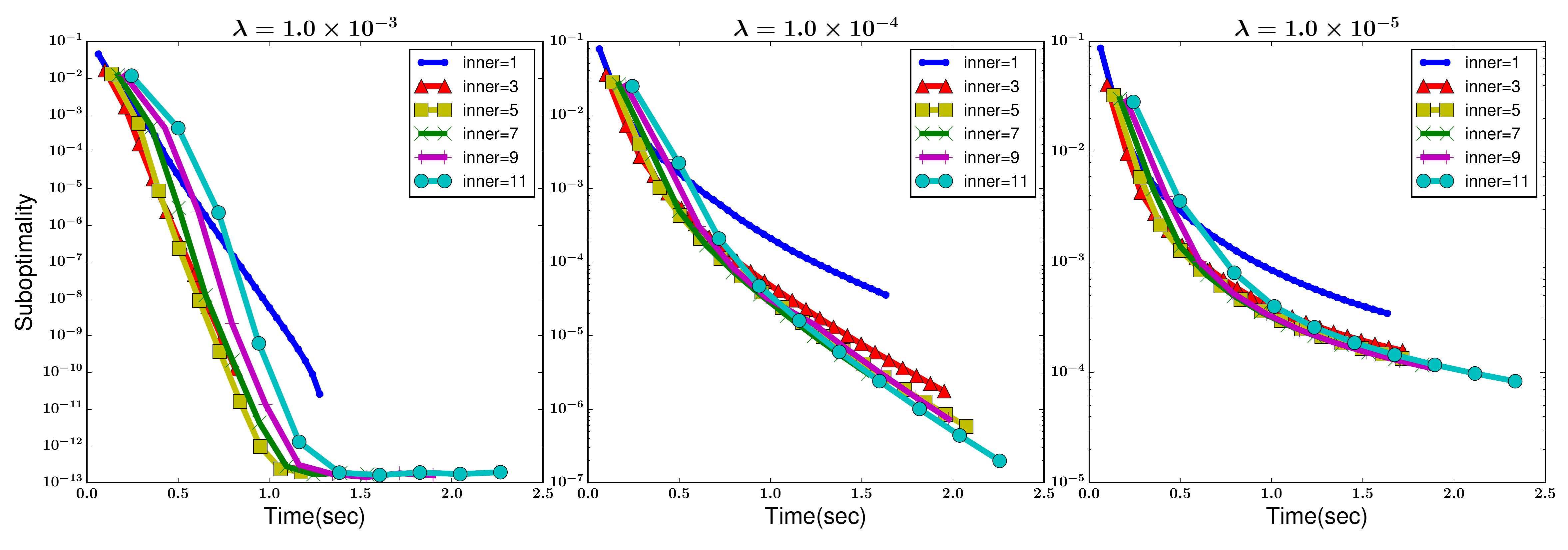}
\caption{Solving $l_1$ logistic regression on Covtype with different inner iteration and different $\lambda$. Unlike the suggestion in (cite SVRG) we fix the very inner iteration in SVRG to be $m=0.01n$ where $n$ is the size of data.\label{Fig:inner}}
\end{figure*} 

The result in Figure~\ref{Fig:inner} shows that our algorithm is quite robust to the choice of number 
of inner iterations. This is a nice advantage in practice when we cannot afford to 
do a grid search for the best hyper-parameters. 
However we also noticed that when $\texttt{inner}=1$ the performance is substantially worse than other choices, this is probably because doing merely one inner iteration cannot solve the subproblem precisely enough to satisfy Lemma \ref{Lemma:precise}.

\par

Next we compare different algorithms on three datasets and choose regularization coefficients from $\lambda\in\{ 1.0\times 10^{-3}, 1.0\times 10^{-4}, 1.0\times 10^{-5} \}$. 
For each combination of $\langle$algorithm, dataset, $\lambda$$\rangle$-triples 
we search the best step size $\eta=10^{-k}, k=0,1,\dots$ (we later found that 
step size is largely determined by $\lambda$ but less depending on data for 
these three datasets, so in fact we are using the 
same step size for all algorithms). 
We choose a fixed number of inner iterations for all $\lambda$ because as the previous experiment shows it won't affect the outcome much. 
The result is shown in Figure \ref{Fig:time}. 
We can see our algorithm outperforms others on Covtype and Mnist with both large and small $\lambda$, furthermore our algorithm is especially good on large regularization where we are expected to see linear or even superlinear convergence. 
On Realsim dataset, our algorithm is slower than LIBLINEAR probably because other algorithms including our fast proximal Newton method as a general purpose algorithm don't exploit the sparse property of data, when $\lambda$ is large our algorithm is comparable to LIBLINEAR. Another key observation is that as the regularization factor $\lambda$ increases, the computational time to convergence decreases (similar conclusion is made in \citep{shi2010fast}). Intuitively a larger regularization $\lambda$ leads to a sparser solution, so if we initialize at $\bw_0=\mathbf{0}$ then it's already close to the optimal solution.
\par
Overall, the experiments show that our algorithm is competitive and slightly better than the state-of-the-art implementation, LIBLINEAR. Therefore, our algorithm not only achieves better theoretical convergence but also has better practical performance than other methods.
\begin{figure*}
\centering
\includegraphics[width=5in]{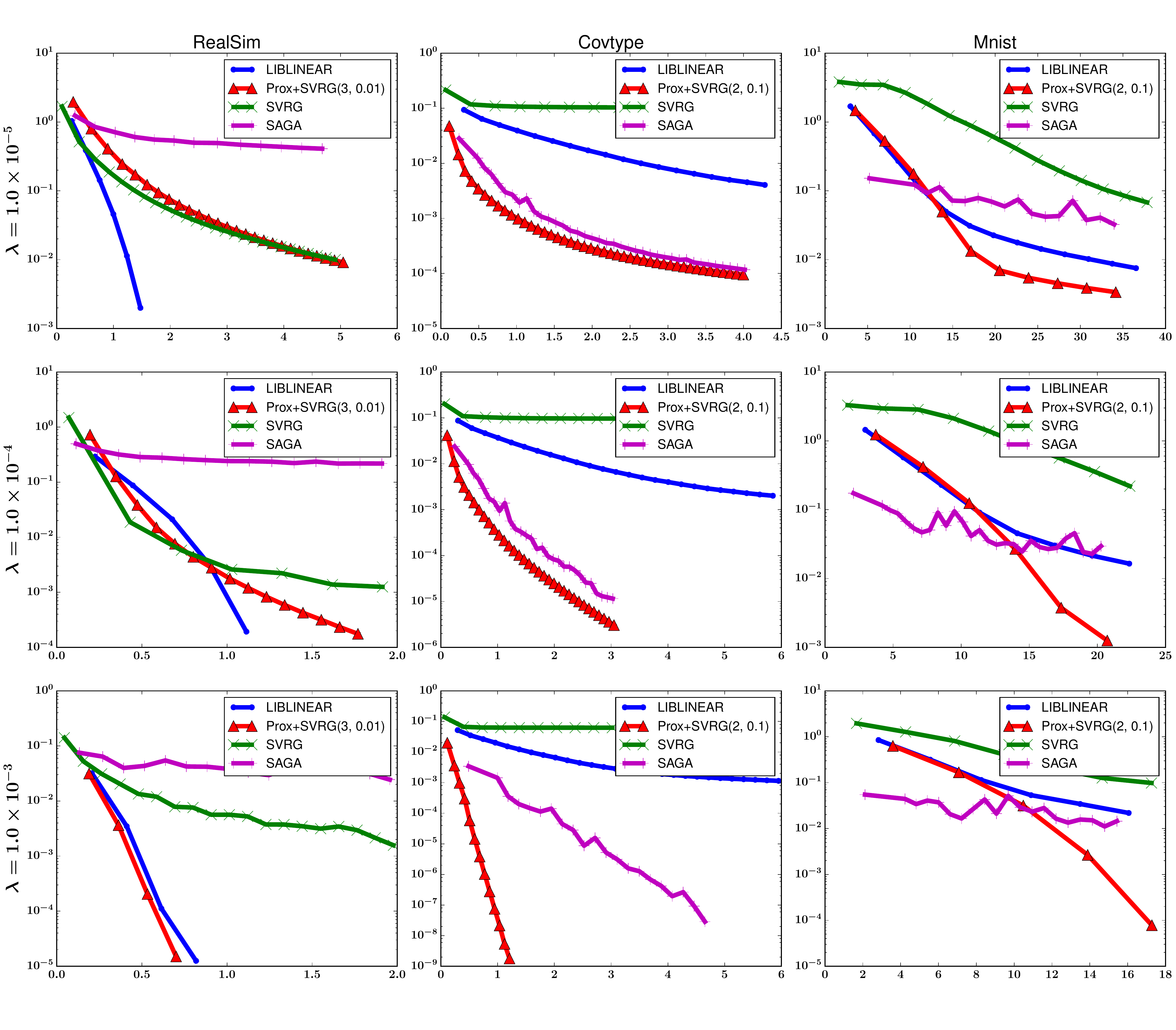}
\caption{Running time comparison of different algorithms on different conditions. Our algorithm is labeled as Prox+SVRG($\texttt{inner}$, $c$) where $\texttt{inner}$ is inner iteration and $c\times n$ is the inner iteration in SVRG.\label{Fig:time}}
\end{figure*}

\section{Summary and discussion}
\label{sec:discussion}

We proposed an inexact subsampled proximal Newton-type method for composite minimization that attains fast rates of convergence. In particular, it matches the computational efficiency of state-of-the-art stochastic first-order methods and LiSSA. At a high-level, the proposed method combines subsampling with accelerated variance reduced first-order methods, and is essentially the composite counterpart to LiSSA. We remark that as long as $n > d$, the proposed method has the best known computational complexity for composite minimization under the stated assumption. The key takeaway is that by leveraging recent advances in stochastic first-order methods, it is possible to design second-order methods that are equally, if not more efficient for large-scale machine learning.




\vskip 0.2in
\bibliography{yuekai,proxnewton}

\newpage
\appendix

\section{Proof of Proposition \ref{pr:dennis_more_deviation}}
\begin{proposition}
If Hessian approximation $B_t$ satisfies Dennis-Mor\'e condition, then $\Vert (B_t-\nabla^2f(w_t))v \Vert^*_{w_t}\le \beta_t\Vert v \Vert_{w_t}$ will also hold.
\end{proposition}

\begin{proof}
Notice that exact Hessian $\nabla^2f(w)=X^\intercal DX$, and $D=\text{diag}\{d_{11}/n,d_{22}/n,\dots, d_{nn}/n\}$. For subsampled Hessian $B_t=\frac{1}{|\mathcal{B}|}\sum_{i\in\mathcal{B}}d_{ii}x_ix_i^\intercal=X^\intercal\sqrt{D}S^\intercal S\sqrt{D}X$, here $S$ is a random diagonal matrix, with each diagonal element $s_{ii}=\frac{1}{p_{ii}}I_i$ and $I_i$ is a i.i.d. random variable:
$$
I_i=
\begin{cases}
1, & p=p_{ii},\\
0, & p=1-p_{ii}.
\end{cases}
$$
i.e. we sample each $(x_i, y_i)$ independently with probability $p_{ii}$. 

Then by expanding the square of the left hand side of deviation condition we have:
\begin{align*}
&w^{\intercal}(B_t - \nabla^2 f(w_t))\nabla^2 f(w_t)^{\dagger}(B_t - \nabla^2 f(w_t))w \\
&\quad= (D^{\frac12}Xw)^{\intercal}(S^{\intercal}S - I_n)D^{\frac12}X(XDX)^{\dagger}D^{\frac12}X(S^{\intercal}S - I_n)(D^{\frac12}Xw) \\
&\quad= (D^{\frac12}Xw)^{\intercal}(S^{\intercal}S - I_n)^2(D^{\frac12}Xw),
\end{align*}
From the first deviation condition:
\begin{equation*}
-\beta_t (D^{\frac{1}{2}}Xw)^\intercal (D^{\frac{1}{2}}Xw)\le w^\intercal (B_t-\nabla^2 f(w_t))w =(D^{\frac12}Xw)^{\intercal}(S^{\intercal}S - I_n)(D^{\frac12}Xw) \le \beta_t (D^{\frac{1}{2}}Xw)^\intercal (D^{\frac{1}{2}}Xw).
\end{equation*}
Because $w$ is arbitrary, suppose $U$ is the space spanned by $\sqrt{D}Xw$: $U=\{x|x=\sqrt{D}Xw\}$, then for any eigenvector $v_i\in U$ of matrix $S^\intercal S-I_n$ the corresponding eigenvalue $\lambda_i$ should lies in $[-\beta_t, \beta_t]$, this ensures:
$$
(D^{\frac12}Xw)^{\intercal}(S^{\intercal}S - I_n)^2(D^{\frac12}Xw) \le \beta_t^2 (D^{\frac{1}{2}}Xw)^\intercal (D^{\frac{1}{2}}Xw),
$$ after rearranging we complete the proof:
$$
\| (B_t-\nabla^2f(w_t))w \|^*_{w_t}\le \beta_t \| w \|_{w_t}
$$

\end{proof}

\section{Proof of Proposition \ref{pr:equivalence}}
\begin{proposition} 
To satisfy Assumption~\ref{ass:dennis-more} it is enough to set $\epsilon=\frac{\beta_t}{1-\beta_t}$ in Theorem \ref{th:sampling-complexity}.
\end{proposition}

\begin{proof}
First of all we can expand Dennis Mor\'e condition:
\begin{equation}
    -\beta_t\Vert v_t \Vert_{w_t}^2 \le v_t^\intercal (B_t-\nabla^2 f(w_t))v_t\le \beta_t\Vert v_t\Vert_{w_t}^2,
\end{equation}
rearranging:
\begin{equation}\label{eq:rearrange-dennis}
(1-\beta_t)\Vert v_t \Vert_{w_t}^2 \le v_t^\intercal B_tv_t \le (1+\beta_t)\Vert v_t\Vert_{w_t}^2,
\end{equation}
so if we set $\epsilon=\frac{\beta_t}{1-\beta_t}$ in Theorem \ref{th:sampling-complexity} and expand the norm then we get:
\begin{equation}\label{eq:norm-convert}
(1-\beta_t)\nabla^2 f(w_t)\preceq B_t\preceq \nabla^2 f(w_t).
\end{equation}
So \eqref{eq:rearrange-dennis} naturally holds.
\end{proof}

\section{Proof of Section \ref{th:converge_phaseI}}
\begin{theorem}
By the update rule of Algorithm \ref{alg:fast_newton_alg} with step size: 
\begin{equation*}
   \eta_t=\frac{1}{1+\beta'_t\tilde{\lambda}_t}, \quad \beta'_t=\frac{1}{\sqrt{1-\beta_t}}, \quad \tilde{\lambda}_t=\Vert w_t^+-w_t \Vert_{B_t},
\end{equation*}
and solve the inner problem with precision $\Vert r_t\Vert_{w_t}^*\le (1-\theta_t)\lambda_t$, where $r_t$ is the subgradient residual:
\begin{equation*}
    \begin{aligned}
    r_t-\nabla f(w_t)-B_t(w_t^+-w_t)\in \partial R(w_t^+),
    \end{aligned}
\end{equation*}
$\theta_t\in(0, 1]$ is a forcing coefficient. Then the function value will decrease by:
\begin{equation}
F(w_{t+1})\le F(w_t) - \eta_t(\theta_t - \beta_t)\lambda_t^2 + \zeta^*(\eta_t\lambda_t).
\end{equation}
\end{theorem}

\begin{proof}
Since $w_{t+1}$ is a convex combination of $w_t$ and $w_t^+$, and $w_{t+1} = w_t + \eta_t(w_t^+ - w_t)$, where $w_t^+$ is the solution of the proximal Newton subproblem, by the convex property of $R(\cdot)$:
\[
R(w_{t+1}) \le (1-\eta_t)R(w_t) + \eta_tR(w_t^+).
\]
Rearranging, 
\[
R(w_{t+1}) - R(w_t) \le \eta_t(R(w_t^+) - R(w_t)).
\]
As long as $\|w_{t+1} - w_t\|_{w_t} < 1$, by the self-concordant property of $f(\cdot)$, we have:
\begin{align*}
 F(w_{t+1}) &\le  F(w_t) + \nabla f(w_t)^{\intercal}(w_{t+1} - w_t) + \zeta^*(\|w_{t+1} - w_t\|_{w_t}) \\
&\quad+ R(w_{t+1}) - R(w_t) \\
&\le  F(w_t) + \eta_t\nabla f(w_t)^{\intercal}(w_t^+ - w_t) + \zeta^*(\eta_t\|w_t^+ - w_t\|_{w_t}) \\
&\quad+ \eta_t(R(w_t^+) - R(w_t)).
\end{align*}
The first inequality is a consequence of \eqref{eq:self-concordance-lipschitz}, and the second inequality is a consequence of the convexity of $R(w)$. We know $R(w_t^+) - R(w_t) \le v_t^{\intercal}(w_t^+ - w_t)$ for any $v_t\in\partial R(w_t^+)$. Consequently,
\begin{equation}
 F(w_{t+1}) \le  F(w_t) + \eta_t\nabla f(w_t)^{\intercal}(w_t^+ - w_t) + \zeta^*(\eta_t\|w_t^+ - w_t\|_{w_t}) + \eta_tv_t^{\intercal}(w_t^+ - w_t).
\label{eq:phit-bound-1}
\end{equation}
We define $\|w_t^+ - w_t\|_{w_t}$ to be the \emph{proximal Newton decrement}, which we denote by $\lambda_t$ hereafter. We observe that the condition $\|w_{t+1} - w_t\|_{w_t} < 1$ is, in terms of the proximal Newton decrement, $\eta_t\lambda_t < 1$.

To further bound $ F(w_{t+1})$, we appeal to the fact that $w_t^+$ is the inexact solution of the proximal quasi-Newton subproblem. By the optimality conditions of the subproblem, we have
\begin{equation}\label{eq:optimal_condition}
r_t-\nabla f(w_t)-B_t(w_t^+-w_t) \in \partial R(w_t^+),
\end{equation}
where $r_t$ is the residual from solving the subproblem inexactly. We require
\begin{equation}
\|r_t\|_{w_t}^* \le (1-\theta_t)\|w_t^+ - w_t\|_{w_t} = (1-\theta_t)\lambda_t,
\label{eq:inexactness}
\end{equation}
where $\theta_t\in (0,1]$ is a \emph{forcing sequence}. By reordering \eqref{eq:optimal_condition} and multiply $w_t^+-w_t$ on both sides yields:
\[
\eta_t r_t^{\intercal}(w_t^+ - w_t) - \eta_t\|w_t^+ - w_t\|_{B_t}^2 = \eta_t(v_t + \nabla f(w_t))^{\intercal}(w_t^+ - w_t),
\]
for some $v_t\in\partial R(w_t)$. Combining with \eqref{eq:phit-bound-1}, we have
\begin{align}
 F(w_{t+1}) &\le  F(w_t) + \zeta^*(\eta_t\lambda_t) - \eta_t\|w_t^+ - w_t\|_{B_t}^2 + \eta_tr_t^{\intercal}(w_t^+ - w_t) \nonumber \\
&\le  F(w_t) + \zeta^*(\eta_t\lambda_t) - \eta_t \|w_t^+ - w_t\|_{B_t}^2 + \eta_t\|r_t\|_{w_t}^*\|w_t^+ - w_t\|_{w_t} \nonumber \\
&\le  F(w_t) + \zeta^*(\eta_t\lambda_t) - \eta_t(\theta_t-\beta_t)\|w_t^+ - w_t\|_{w_t}^2,
\label{eq:phit-bound-2}
\end{align}
where the third step is a consequence of \eqref{eq:inexactness}, we also convert $\|\cdot\|_{B_t}$ to $\|\cdot\|_{w_t}$ by \eqref{eq:norm-convert}.

To complete the proof, we pick $\eta_t$ to ensure $\eta_t\lambda_t < 1$. 
\citep{tran2015composite} propose 
\begin{equation}
\eta_t = \frac{\lambda_t^2}{\lambda_t(\lambda_t + \lambda_t^2)},\quad\lambda_t = \|w_t^+ - w_t\|_{w_t}.
\label{eq:tran-dinh-step-length}
\end{equation}
Unfortunately, evaluating \eqref{eq:tran-dinh-step-length} involves evaluating the (exact) proximal Newton decrement, which is impractical.
We propose
\begin{equation}
\eta_t = ({\textstyle 1 + \frac{1}{\sqrt{1 - \beta_t}}\tilde{\lambda}_t})^{-1}.
\label{eq:step-length}
\end{equation}
As long as we have the deviation condition
\begin{equation}
|v_1^{\intercal}(B_t - \nabla^2 f(w_t))v_2| \le \beta_t\|v_1\|_{w_t}\|v_2\|_{w_t},
\label{eq:cn:deviation-condition}
\end{equation}
for any $v_1,v_2\in\textsf{cone}(\mathbb{B}_1^d - w_t)$, it is easy to show $\frac{1}{\sqrt{1 - \beta_t}}\tilde{\lambda}_t > \lambda_t$, which in turn ensures
\[
\eta_t\lambda_t = \lambda_t({\textstyle 1 + \frac{1}{\sqrt{1 - \beta_t}}\tilde{\lambda}_t})^{-1} \le \frac{\lambda_t}{1 + \lambda_t} < 1.
\]

\end{proof}

\section{Proof of Corollary \ref{co:fix_decrement}}
\begin{corollary}
By fixing $\beta_t<\min\{\theta_t, \frac{1}{3}\}$ and step size $\eta_t = \frac{\theta_t - \beta_t}{1+\beta_t'(\theta_t-\beta_t)\tilde{\lambda}_t}<\frac{1}{1+\beta_t'\tilde{\lambda}_t}$ then the decrement of function value in each step is at least $(\frac{1}{2(1-\beta_t)}-\frac{2\beta_t}{1-\beta_t^2})\eta_t(\theta_t-\beta_t)\tilde{\lambda}_t^2$ which is bounded away from zero as long as $\lambda_t\ge\bar{\lambda}$. So within finite steps, the iterates will enter into $\lambda_t< \bar{\lambda}$.
\end{corollary}

\begin{proof}
  By the restricted Dennis-Mor\'{e} condition and the choice of $\eta_t$, we have
  \begin{align*}
  F(w_t)-F(w_{t+1}) &\ge \eta_t(\theta_t-\beta_t)\|w_t^+-w_t\|_{w_t}^2 - \zeta^*(\eta_t\|w_t^+-w_t\|_{w_t}) \\
  &\ge \frac{\eta_t(\theta_t-\beta_t)}{1+\beta_t}\|w_t^+-w_t\|_{B_t}^2-\zeta^*(\eta_t\beta_t'\|w_t^+-w_t\|_{B_t})\\
  &\textstyle=\frac{1}{1+\beta_t} \frac{(\theta_t-\beta_t)^2\tilde{\lambda}_t^2}{1+\beta_t'(\theta_t-\beta_t)\tilde{\lambda}_t}+\frac{(\theta_t-\beta_t)\beta_t'\tilde{\lambda}_t}{1+\beta_t'(\theta_t-\beta_t)\tilde{\lambda}_t}-\log(1+\beta_t'(\theta_t-\beta_t)\tilde{\lambda}_t)\\
  &\textstyle= \frac{-2\beta_t}{1-\beta_t^2}\frac{(\theta_t-\beta_t)^2\tilde{\lambda}_t^2}{1+\beta_t'(\theta_t-\beta_t)\tilde{\lambda}_t}+\beta_t'(\theta_t-\beta_t)\tilde{\lambda}_t-\log(1+\beta_t'(\theta_t-\beta_t)\tilde{\lambda}_t),
  \end{align*}
  recalling the inequality $x - \log(1+x)\ge \frac{\frac{x^2}{2}}{1+x}$, we have
  \begin{align*}
  F(w_t)-F(w_{t+1}) &\textstyle\ge
  (\frac{1}{2}\beta_t'^2-\frac{2\beta_t}{1-\beta_t^2})\frac{(\theta_t-\beta_t)^2\tilde{\lambda}_t^2}{1+\beta_t'(\theta_t-\beta_t)\tilde{\lambda}_t}\\
  &\textstyle= (\frac{1}{2(1-\beta_t)}-\frac{2\beta_t}{1-\beta_t^2})\eta_t(\theta_t-\beta_t)\tilde{\lambda}_t^2.
  \end{align*}
  The condition $\beta_t < \frac13$ ensures $\frac{1}{2(1-\beta_t)}-\frac{2\beta_t}{1-\beta_t^2}>0$. The Phase I analysis implies as long as $\lambda_t \ge \bar{\lambda}$ for some $\bar{\lambda}>0$, the $t$-th iteration decreases the cost function by at least
  \[
  \Big(\frac{1}{2(1-\beta_t)}-\frac{2\beta_t}{1-\beta_t^2}\Big)\eta_t(\theta_t-\beta_t)\Big(\frac{\bar{\lambda}}{\beta_t'}\Big)^2.
  \]
  Thus, as long as the cost is bounded below, we will reach $\lambda_t < \bar{\lambda}$ after finitely many iterations.
\end{proof}

\section{Proof of Theorem \ref{th:linear-quadratic}}
\begin{theorem}
When $\lambda_t<\bar{\lambda}$ and the subproblem solver yields a solution such that the subgradient residual  $\Vert r_t\Vert_{w_t}^*\le (1-\theta_t)\lambda_t$, $\theta_t\in(0, 1]$ then Newton decrement $\lambda_t$ will converge to zero linear-quadratically:
\begin{equation}
\lambda_{t+1}\le \frac{\frac{\theta_t-\beta_t}{\theta_{t+1}-\beta_{t+1}}\lambda_t^2+\frac{1+\beta_t-\theta_t}{\theta_{t+1}-\beta_{t+1}}\lambda_t}{(1-\lambda_t)^2}
\end{equation}
for inexact solution $\beta_t\ne 0$, $\theta_t\ne 1$ and $\lambda_t$ is small enough, the numerator of RHS will be dominated by $\frac{1+\beta_t-\theta_t}{\theta_{t+1}-\beta_{t+1}}\lambda_t$ so the geometric factor is $\rho=\frac{1+\beta_t-\theta_t}{\theta_{t+1}-\beta_{t+1}}$ asymptotically.
\end{theorem}

\begin{proof}
Recall the proximal Newton subproblem 
\[
w_t^+ \approx \argmin_{w\in\mathbb{R}^d}\nabla f(w_t)^{\intercal}(w - w_t) + \frac12\|w - w_t\|_{B_t}^2 +  R(w),
\]
and its first-order optimality condition 
\begin{equation}
r_t = \nabla f(w_t) + B_t(w_t^+ - w_t) + v_t^+,
\label{eq:proxnewton-subproblem-optimality-cnd}
\end{equation}
where $v_t^+\in\partial R(w_t^+)$. We see that $w_t^+$ satisfies the optimality condition of the (unsketched) proximal Newton method inexactly:
\[
0 = \underbrace{\nabla f(w_t) + \nabla^2 f(w_t)(w_t^+ - w_t)}_{\nabla\psi_t(w_t^+)} + \underbrace{(B_t - \nabla^2 f(w_t))(w_t^+ - w_t) - r_t}_{r'_t} + v_t^+.
\]
By the convexity of $R(\cdot)$,
\begin{equation}
\begin{aligned}
&(-\nabla\psi_{t+1}(w_{t+1}^+) - r'_{t+1} + \nabla\psi_t(w_t^+) + r'_t)^{\intercal}(w_{t+1}^+ - w_{t+1}) \\
&\quad= (-\nabla\psi_{t+1}(w_{t+1}^+) - r'_{t+1} + \nabla\psi_t(w_t^+) + r'_t)^{\intercal}(w_{t+1}^+ - w_t^+) \\
&\quad= (v_{t+1} - v_t)^{\intercal}(w_{t+1}^+ - w_t^+) \\
&\quad \ge 0,
\end{aligned}
\label{eq:proxnewton-step-monotone}
\end{equation}
which leads to a bound on $\lambda_{t+1}^2$:
\begin{align*}
\lambda_{t+1}^2 &= \|w_{t+1}^+ - w_{t+1}\|_{w_{t+1}}^2 \\
&\le \|w_{t+1}^+ - w_{t+1}\|_{w_{t+1}}^2 + 2(-\nabla\psi_{t+1}(w_{t+1}^+) -r'_{t+1} + \nabla\psi_t(w_t^+) + r'_t)^{\intercal}(w_{t+1}^+ - w_{t+1}) \\
&\quad+ \|\nabla^2 f(w_{t+1})^{-1}(-\nabla\psi_{t+1}(-w_{t+1}^+) - r'_{t+1} + \nabla\psi_t(w_t^+) + r'_t)\|_{w_{t+1}}^2 \\
&= \|w_{t+1}^+ - w_{t+1} + \nabla^2 f(w_{t+1})^{-1}(-\nabla\psi_{t+1}(w_{t+1}^+) - r'_{t+1} + \nabla\psi_t(w_t^+) + r'_t)\|_{w_{t+1}}^2.
\end{align*}
Equivalently,
\[
\lambda_{t+1} \le \|w_{t+1}^+ - w_{t+1} + \nabla^2 f(w_{t+1})^{-1}(-\nabla\psi_{t+1}(w_{t+1}^+) - r'_{t+1} + \nabla\psi_t(w_t^+) + r'_t)\|_{w_{t+1}}.
\]
By the definition of $\nabla\psi_{t+1}(w_{t+1}^+)$, we have
\[
-\nabla^2 f(w_{t+1})^{-1}\nabla\psi_{t+1}(w_{t+1}^+) = -\nabla^2 f(w_{t+1})^{-1}\nabla f(w_{t+1}) - (w_{t+1}^+ - w_{t+1}).
\]
Plugging the preceding expression into the bound on $\lambda_{t+1}$, we obtain
\begin{align*}
\lambda_{t+1} &\le \|\nabla^2 f(w_{t+1})^{-1}(-\nabla f(w_{t+1}) - r'_{t+1} + \nabla\psi_t(w_t^+) + r'_t)\|_{w_{t+1}} \\
&= \|-\nabla f(w_{t+1}) - r'_{t+1} + \nabla\psi_t(w_t^+) + r'_t\|_{w_{t+1}}^* \\
&\le \|-\nabla f(w_{t+1}) + \nabla\psi_t(w_t^+) + r'_t\|_{w_{t+1}}^* + \|r'_{t+1}\|_{w_{t+1}}^*.
\end{align*}
Recalling the definition of $r'_{t+1}$ and \eqref{eq:inexactness}, we have:
\begin{align*}
\lambda_{t+1} &\le \|-\nabla f(w_{t+1}) + \nabla\psi_t(w_t^+) + r'_t\|_{w_{t+1}}^* + (1-\theta_{t+1})\|w_{t+1}^+ - w_{t+1}\|_{w_{t+1}} \\
&\qquad+ \|(B_{t+1} - \nabla^2 f(w_{t+1}))(w_{t+1}^+ - w_{t+1})\|_{w_{t+1}}^*.
\end{align*}
By the second deviation condition, we have:
\[
\|(B_{t+1} - \nabla^2 f(w_{t+1}))(w_{t+1}^+ - w_{t+1})\|_{w_{t+1}}^* \le \beta_{t+1}\lambda_{t+1}.
\]
further we rearrange to obtain
\[
(\theta_{t+1} - \beta_{t+1})\lambda_{t+1} \le \|-\nabla f(w_{t+1}) + \nabla\psi_t(w_t^+) + r'_t\|_{w_{t+1}}^*.
\]
then converting the norm from $\|\cdot\|_{w_{t+1}}^*$ to $\|\cdot\|_{w_t}^*$ by \eqref{eq:self-concordance-hessian}:
\begin{align*}
&\|-\nabla f(w_{t+1}) + \nabla\psi_t(w_t^+) + r'_t\|_{w_{t+1}}^* \\
&\quad\le \frac{\|-\nabla f(w_{t+1}) + \nabla\psi_t(w_t^+) + r'_t\|_{w_t}^*}{1 - \|w_{t+1} - w_t\|_{w_t}} \\
&\quad\le \frac{\|-\nabla f(w_{t+1}) + \nabla\psi_t(w_t^+)\|_{w_t}^* + \|r'_t\|_{w_t}^*}{1 - \|w_{t+1} - w_t\|_{w_t}}.
\end{align*}
Focusing on controlling the numerator, by \eqref{eq:self-concordance-gradient}, we have
\begin{align*}
&\|-\nabla f(w_{t+1}) + \nabla\psi_t(w_t^+)\|_{w_t}^* \\
&\quad= \|-\nabla f(w_{t+1}) + \nabla f(w_t) + \nabla^2 f(w_t)(w_{t+1} - w_t)\|_{w_t}^* \\
&\quad\le \frac{\|w_{t+1} - w_t\|_{w_t}^2}{1 - \|w_{t+1} - w_t\|_{w_t}} = \frac{\lambda_t^2}{1 - \lambda_t}.
\end{align*}
We control $\|r'_t\|_{w_t}^*$ in the exact same way we controlled $\|r'_{t+1}\|_{w_{t+1}}^*$:
\begin{align*}
\|r'_t\|_{w_t}^* &\le \|r_t\|_{w_t}^* + \|(B_t - \nabla^2 f(w_t))(w_t^+ - w_t)\|_{w_t}^* \\
&\le  (1-\theta_t)\lambda_t +  \beta_t\lambda_t,
\end{align*}
where the second inequality is a consequence of Assumption \ref{ass:dual_norm} and \eqref{eq:inexactness}. Consequently,
\begin{align*}
&\frac{\|-\nabla f(w_{t+1}) + \nabla\psi_t(w_t^+)\|_{v_1}^* + \|r'_t\|_{w_t}^*}{1 - \|w_{t+1} - w_t\|_{w_t}} \\
&\quad\le \frac{\lambda_t^2}{(1 - \lambda_t)^2} + \frac{(1+\beta_t - \theta_t)\lambda_t}{1 - \lambda_t} = \frac{(1 + \beta_t - \theta_t)\lambda_t+ (\theta_t - \beta_t)\lambda_t^2}{(1 - \lambda_t)^2}. 
\end{align*}
In summary, we have
\[
(\theta_{t+1}-\beta_{t+1})\lambda_{t+1} \le \frac{(1 + \beta_t - \theta_t)\lambda_t+ (\theta_t - \beta_t)\lambda_t^2}{(1 - \lambda_t)^2}.
\]
We divide by $\theta_{t+1}-\beta_{t+1}$ to obtain
\begin{equation}
\lambda_{t+1} \le \frac{\frac{\theta_t-\beta_t}{\theta_{t+1} - \beta_{t+1}}\lambda_t^2 + \frac{1 + \beta_t - \theta_t}{\theta_{t+1} - \beta_{t+1}}\lambda_t}{(1 - \lambda_t)^2},
\label{eq:decrement-bound}
\end{equation}
\end{proof}

\section{Proof of Corollary \ref{co:lambda_subopt}}
\begin{corollary}
If $\lambda_t<\min\{\bar{\lambda}, \frac{1}{2-\theta_t}\}$ and use the undamped update: $w_{t+1}=w^+$ then the function value to minimum is upper bounded by:
\begin{equation}
F(w_t^+)-F(w^*) \le \lambda_t^2
\end{equation}
it's easy to see that the $\text{LHS}\to 0$ as $\lambda_t\to 0$. Practically we use $\tilde{\lambda}_t$ to replace $\lambda_t$, this is validated by Dennis-Mor\'e condition \ref{ass:dennis-more}.
\end{corollary}
\begin{proof}
For any $w^*, w_t, w_t^+\in \mathsf{dom}f$, we have from \eqref{eq:self-concordance-lipschitz}:
\begin{equation}
\begin{aligned}
 f(w^*) &\ge f(w_t)+\nabla f(w_t)^\intercal (w^*-w_t)+\zeta(\Vert w^*-w_t\Vert_{w_t})\\
  &\ge f(w_t^+)-\nabla f(w_t)^\intercal(w_t^+-w_t)+\nabla f(w_t)^\intercal(w^*-w_t)\\
  &\quad +\zeta(\Vert w_t-w^* \Vert_{w_t})-\zeta^*(\Vert w_t^+-w_t \Vert_{w_t})
\end{aligned}
\end{equation}
and by convexity of $R(\cdot)$:
\begin{equation}
R(w^*)\ge R(w_t^+)+\partial R(w_t^+)^\intercal (w^*-w_t^+)
\end{equation}
as well as the definition of gradient residual $r_t$:
\begin{equation}
r_t-\nabla f(w_t)-B_t(w_t^+-w_t)\in \partial R(w_t^+)
\end{equation}
combining three inequalities above we get:
\begin{equation}
\begin{aligned}
F(w^*)&\ge F(w_t^+)+(r_t-B_t(w_t^+-w_t))^\intercal (w^*-w_t)-r_t^\intercal(w_t^+-w_t)\\
&\quad +\Vert w_t^+-w_t \Vert_{B_t}^2+\zeta(\Vert w_t-w^* \Vert_{w_t})-\zeta^*(\Vert w_t^+-w_t \Vert_{w_t})
\end{aligned}
\end{equation}
For simplicity, set $t=\Vert w_t-w^*\Vert_{w_t}$ and recall $\Vert r_t \Vert_{w_t}^*\le (1-\theta_t)\lambda_t$:
\begin{equation}
\begin{aligned}
F(w^*)&\ge F(w_t^+)-\Vert r_t-B_t(w_t^+-w_t)\Vert_{w_t}^*\Vert w_t-w^*\Vert_{w_t}-\Vert r_t\Vert_{w_t}^*\Vert w_t^+-w_t\Vert+(1-\beta_t)\lambda_t^2\\
&\quad +\zeta(t)-\zeta^*(\lambda_t)\\
&\ge F(w_t^+)-(\Vert r_t\Vert_{w_t}^*+\Vert B_t(w_t^+-w_t)\Vert_{w_t}^*)t-(1-\theta_t)\lambda_t^2+(1-\beta_t)\lambda_t^2+\zeta(t)-\zeta^*(\lambda_t)\\
&\ge F(w_t^+)-(2+\beta_t-\theta_t)\lambda_tt+(\theta_t-\beta_t)\lambda_t^2+\zeta(t)-\zeta^*(\lambda_t)\\
\end{aligned}
\end{equation}
here we use the fact that $(1-\beta_t)\nabla^2 f(w_t)\preceq B_t\preceq (1+\beta_t)\nabla^2f(w_t)$, now by maximize $t\in \mathbb{R}^+$ in right hand side we can bound the function value to the minimum:
\begin{equation}
\begin{aligned}
F(w_t^+)-F(w^*)&\le \zeta^*((2+\beta_t-\theta_t)\lambda_t)+\zeta^*(\lambda_t)-(\theta_t-\beta_t)\lambda_t^2\\
&\overset{!}{\le} \lambda_t^2,
\end{aligned}
\end{equation}
which is attained at $t=t^*$ and:
\begin{equation}\label{eq:best_t}
\frac{1}{1+t^*}=1-(2+\beta_t-\theta_t)\lambda_t.
\end{equation}
The inequality $\overset{!}{\le}$ comes from $\zeta^*(\lambda_t)\le \lambda_t^2$ for $\lambda_t\in [0, 0.68)$ and $\zeta^*((2+\beta_t-\theta_t)\lambda_t)\le (\theta_t-\beta_t)\lambda_t^2$ for $\theta_t\in (0.764+\beta_t, 1]$.
\end{proof}

\section{Proof of Lemma \ref{Lemma:precise}}
Now we want to prove that if the subproblem solution is $\epsilon_t$-suboptimal and $\epsilon_t$ converges to 0 exponentially then the proximal Newton method will converge to optima. Recall that we want to make sure:
\begin{equation}\label{bound_rt}
\Vert r_t\Vert_{w_t}^*\le \theta_t\lambda_t
\end{equation}
where we assume that $\theta_t\ge c>0$ is constant(at least bounded) and $\lambda_t$ converges to 0 exponentially. And the subproblem is :
\begin{equation}
f^{\text{sub}}_t(w)=\nabla f(w_t)^\intercal (w-w_t)+\frac{1}{2}(w-w_t)^\intercal B_t(w-w_t)+R(w)
\end{equation}
Suppose the $n$-th inner iteration for $f^{\text{sub}}_t(w)$ is $\tilde{w}_n$, \textbf{imagine} we do one extra step of proximal gradient based on $\tilde{w}_n$ with step size $1/L$ then we have:
\begin{equation}
f^{\text{sub}}_t(\tilde{w}_n)-f_t^* \ge f^{\text{sub}}_t(\tilde{w}_n)-f^{\text{sub}}_t(\tilde{w}_n^+) \ge \frac{L}{2}\Vert \tilde{w}_n^+-\tilde{w}_n \Vert^2
\end{equation}
where $f_t^*=\min_w f^{\text{sub}}_t(w)$ and $w_n^+=\mathsf{prox}_{R(\cdot)/L}(w_n-\frac{1}{L}\nabla f^{\text{sub}}_t(\tilde{w}_n))$ is the proximal gradient update, which is equivalent to:
\begin{equation}
\begin{aligned}
&L(\tilde{w}_n-\tilde{w}_n^+)\in \partial R(\tilde{w}_n^+)+\nabla f(w_t)+B_t(\tilde{w}_n-w_t)\\
\Leftrightarrow & (LI-B_t)(\tilde{w}_n-\tilde{w}_n^+)\in \partial R(\tilde{w}_n^+)+\nabla f(w_t)+B_t(\tilde{w}_n^+-w_t)
\end{aligned}
\end{equation}
comparing with the definition of residual $r_t$ we know:
\begin{equation}
r_t=(L\cdot\mathbb{I}_d-B_t)(\tilde{w}_n-\tilde{w}_n^+)
\end{equation}
where $\mathbb{I}_d$ is the $d\times d$ identity matrix. Combing those relations above,
\begin{equation}
\Vert r_t\Vert_{w_t}^* \le \frac{1}{\sqrt{\mu}}\Vert r_t\Vert_2 \le \frac{L-\mu}{\sqrt{\mu}}\Vert \tilde{w}_n-\tilde{w}_n^+\Vert_2\le \sqrt{\frac{2(L^2-\mu^2)}{\mu L}(f^{\text{sub}}_t(\tilde{w}_n)-f_t^*)}
\end{equation}
to make sure \eqref{bound_rt} holds, it is enough to solve the subproblem to:
\begin{equation}
f^{\text{sub}}_t(\tilde{w}_n)-f_t^*\le \frac{\mu L}{2(L^2-\mu^2)}(\theta_t\lambda_t)^2
\end{equation}
and since in the phase-II of proximal Newton method, $\lambda_t$ converges to 0 exponentially, then $\epsilon_t=f^{\text{sub}}_t(\tilde{w}_n)-f_t^*$ also converges to 0 exponentially.

\section{Proof of Lemma \ref{le:warm_start}}
\begin{lemma}
Let 
$$f^{\text{sub}}_t(w)=\nabla f^{\intercal}(w_t)+\frac{1}{2}(w-w_t)^{\intercal}B_t(w-w_t) + R(w)$$
and $w_{t}^+$ is the $\epsilon_t$-inexact solution of $f_t(w)$, i.e.:
\begin{equation}
\epsilon_t\le \frac{\mu L}{2(L^2-\mu^2)}((1-\theta_t)\lambda_t)^2
\end{equation}
then we have:
$$
f^{\text{sub}}_{t+1}(w_t^+)-f_{t+1}^*\le c\cdot \epsilon_t=\mathcal{O}(\lambda_{t+1}^2)
$$
\end{lemma}

\begin{proof}
In phase-II we have:
\begin{equation}
\begin{aligned}
    f^{\text{sub}}_t(w)&=\nabla f^{\intercal}(w_t)(w-w_t)+\frac{1}{2}(w-w_t)^{\intercal}B_t(w-w_t)+R(w)\\
    f^{\text{sub}}_{t+1}(w)&=\nabla f^{\intercal}(w_{t+1})(w-w_{t+1})+\frac{1}{2}(w-w_{t+1})^{\intercal}B_{t+1}(w-w_{t+1})+R(w),\\
\end{aligned}
\end{equation}
and $f^{\text{sub}}_t(w_{t+1})-f_t^*\le \epsilon_t$. So we hope:
$$
f^{\text{sub}}_{t+1}(w_{t+1})-f_{t+1}^*=R(w_{t+1})-f_{t+1}^*=O(\epsilon_{t+1}),
$$
because $f^{\text{sub}}_{t+1}(w_{t+1}^+)-f_{t+1}^*=\epsilon_{t+1}$ so we only need to prove $f^{\text{sub}}_{t+1}(w_{t+1})-f^{\text{sub}}_{t+1}(w_{t+1}^+)=O(\epsilon_{t+1})$. Indeed we have:
\begin{equation}
\begin{aligned}
f^{\text{sub}}_{t+1}(w_{t+1})-f^{\text{sub}}_{t+1}(w_{t+1}^+)&= R(w_{t+1})-R(w_{t+1}^+)-\nabla f^{\intercal}(w_{t+1})(w_{t+1}^+-w_{t+1})\\
&-\frac{1}{2}(w_{t+1}^+-w_{t+1})^{\intercal}B_{t+1}(w_{t+1}^+-w_{t+1}),
\end{aligned}
\end{equation}
and suppose $v_t\in\partial R(w_t)$, we have:
\begin{equation}
\begin{aligned}
& R(w_{t+1}) - R(w_{t+1}^+)-\nabla f^{\intercal}(w_{t+1})(w_{t+1}^+-w_{t+1})\\
&\le -(v_{t+1}+\nabla f(w_{t+1}))^{\intercal}(w_{t+1}^+-w_{t+1})\\
&\le\Vert v_{t+1}+\nabla f(w_{t+1}) \Vert_{w_{t+1}}^*\Vert w_{t+1}^+-w_{t+1} \Vert_{w_{t+1}},
\end{aligned}
\end{equation}
from the definition of $r_t$:
\begin{equation}
r_t \in \nabla f(w_t) + B_t(w_{t+1}-w_t)+v_{t+1},
\end{equation}
we have:
\begin{equation}
\begin{aligned}
&\Vert v_{t+1} + \nabla f(w_{t+1}) \Vert_{w_{t+1}}^*\\
&\le \frac{\Vert v_{t+1} + \nabla f(w_{t+1})\Vert_{w_t}^*}{1-\Vert w_{t+1}-w_t\Vert_{w_t}^2}\\
&=\frac{\Vert\nabla f(w_{t+1})-\nabla f(w_t)-B_t(w_{t+1}-w_t)+r_t\Vert_{w_{t}}^*}{1-\Vert w_{t+1}-w_t\Vert_{w_t}^2}\\
&\le \frac{\Vert\nabla f(w_{t+1})-\nabla f(w_t)-B_t(w_{t+1}-w_t)\Vert_{w_{t}}^*+(1-\theta_t)\lambda_t}{1-\Vert w_{t+1}-w_t\Vert_{w_t}^2},\\
\end{aligned}
\end{equation}
and using properties of self-concordant function:
\begin{equation}
\begin{aligned}
&\Vert\nabla f(w_{t+1})-\nabla f(w_t)-B_t(w_{t+1}-w_t)\Vert_{w_{t}}^*\\
&\le \Vert\nabla f(w_{t+1})-\nabla f(w_t)-\nabla^2 f(w_t)(w_{t+1}-w_t)\Vert_{w_{t}}^*\\
&+\Vert (\nabla^2 f(w_t)-B_t)(w_{t+1}-w_t)\Vert_{w_t}^*\\
&\le \frac{\Vert w_{t+1}-w_t\Vert_{w_t}^2}{1-\Vert w_{t+1}-w_t\Vert_{w_t}}+\beta_t\Vert w_{t+1}-w_t\Vert_{w_t}.
\end{aligned}
\end{equation}
So we have proven that $\epsilon_0=\mathcal{O}(\lambda_t^2)$ so $\frac{\epsilon_0}{\epsilon}$ is bounded by some constants.
\end{proof}


\end{document}